\documentclass{article}

\usepackage{PRIMEarxiv}

\usepackage[utf8]{inputenc} 
\usepackage[T1]{fontenc}    
\usepackage{hyperref}       
\usepackage{url}            
\usepackage{booktabs}       
\usepackage{amsfonts}       
\usepackage{nicefrac}       
\usepackage{microtype}      
\usepackage{lipsum}
\usepackage{fancyhdr}       
\usepackage{graphicx}       
\graphicspath{{media/}}     

\usepackage{algorithm}
\usepackage{algorithmic}
\usepackage{amsmath}
\usepackage{mathtools}

\numberwithin{equation}{section}
\usepackage{amssymb}
\usepackage{bbm}
\usepackage{enumitem}
\usepackage{amsthm}
\usepackage{verbatim}
\usepackage{biblatex} 

\usepackage{amsthm}

\theoremstyle{definition}

\newtheorem{definition}{Definition}[section]
\newtheorem{theorem}[definition]{Theorem}
\newtheorem{lemma}[definition]{Lemma}

\title{Hellinger-UCB: A novel algorithm for stochastic multi-armed bandit problem and cold start problem in recommender system}

\author{
  Ruibo Yang \\
  JPMorgan Chase \\
  New York, NY, USA\\
  \texttt{ruyyang0216@gmail.com} \\
   \And
  Jiazhou Wang \\
  Meta \\
  Menlo Park, CA, USA\\
  \texttt{jiazhouwang@meta.com} \\
   \And
  Andrew Mullhaupt \\
  Stony Brook University \\
  Stony Brook, NY, USA\\
  \texttt{doc@zen-pharaohs.com} \\
}

\begin{document}
\maketitle

\begin{abstract}
In this paper, we study the stochastic multi-armed bandit problem, where the reward is driven by an unknown random variable. We propose a new variant of the Upper Confidence Bound (UCB) algorithm called Hellinger-UCB, which leverages the squared Hellinger distance to build the upper confidence bound. We prove that the Hellinger-UCB reaches the theoretical lower bound. We also show that the Hellinger-UCB has a solid statistical interpretation. We show that Hellinger-UCB is effective in finite time horizons with numerical experiments between Hellinger-UCB and other variants of the UCB algorithm. As a real-world example, we apply the Hellinger-UCB algorithm to solve the cold-start problem for a content recommender system of a financial app. With reasonable assumption, the Hellinger-UCB algorithm has convenient but important lower latency feature. The online experiment also illustrates that the Hellinger-UCB outperforms both KL-UCB and UCB1 in the sense of a higher click-through rate (CTR).
\end{abstract}

\keywords{Multi-armed bandit problem, Upper Confidence Bound, Squared Hellinger Distance, Cold-start problem}

\section{Introduction}
\subsection{Stochastic multi-armed bandit problem}
The stochastic multi-armed bandit problem(MAB) (\cite{Robbins1952SomeAO}) is a sequential decision problem defined by a payoff function and a set of actions. At each step $t\in{1,2,..., T}$, an action $A_{t}$ is chosen from the action set $A=\{1,2,3,...K\}$ by the agent. And the associated reward $r_{t}(A_{t})$, which is independent and identically distributed(i.i.d.), is obtained. The goal of the agent is to find the optimal strategy that maximizes the cumulative payoff obtained in a sequence of decisions
\begin{equation}
   S_{A_{t}}(T)=\sum_{s=1}^{T}r_{t}(A_{t}).
\end{equation}
The agent has to choose the action it has experienced to gain a reward, but the collected information may not be sufficient to make optimal decisions; on the other hand, it needs to try new actions to collect more information and make a better decision in the future. But the new action can be sub-optimal and the agent has to pay the cost of collecting information. The agent must come up with a strategy to maximize the cumulative payoff by dealing with the dilemma between exploitation and exploration. The pseudo-regret $\bar{R}_{T}$ is introduced to evaluate the performance of a strategy. It is defined as the maximized expectation of the difference between the cumulative payoff of consistently choosing the best action and that of the strategy in the first $T$ steps
\begin{equation}
   \bar{R}_{T}=\max_{i=1,2,...,K}\mathbb{E}[\sum_{s=1}^{T}r_{t}(i)-\sum_{s=1}^{T}r_{t}(A_{t})]
\end{equation}
Lai and Robbins(1985)\cite{lai1985asymptotically} showed that if for all $\epsilon>0$ and $ \theta,\theta_{i}\in\Theta$ with $\mu(\theta)>\mu(\theta_{i})$, there exists $\delta>0$ such that $|KL(\theta, \theta_{i})-KL(\theta, \theta_{j})|<\epsilon$ whenever $\mu(\theta_{i})<\mu(\theta_{j})<\mu(\theta_{i})+\delta$, and the following theorem is true. Let $N_{i}(t)$ denote the number of times the agent selected action $i$ in the first $t$ steps,
\begin{theorem}
If a policy has regret $\bar{R}_{T}=o(T^{a})$ for all $a>0$ as $T\to0$, the number of draws up to time $t$, $N_{i}(t)$ of any sub-optimal arm $i$ is lower bounded
\begin{equation}
    \lim\inf_{T\to\infty}\frac{N_{i}(T)}{\log(T)}\geqslant\frac{1}{\inf_{\theta\in\Theta_{i}:\mathbb{E}_{\theta}>\mu^{*}} KL(\theta_{i}, \theta)}
\end{equation}
Therefore, the regret is lower-bounded
\begin{equation}
    \lim\inf_{T\to\infty}\frac{\bar{R}_{T}}{\log(T)}\geqslant\sum_{i:\Delta_{i}\geqslant0}\frac{\Delta_{i}}{\inf_{\theta\in\Theta_{i}:\mathbb{E}_{\theta}>\mu^{*}} KL(\theta_{i}, \theta)}
\end{equation}.
\end{theorem}
Burnetas and Katehakis(1996)\cite{burnetas1996optimal} extended this result to nonparametric cases.

The stochastic multi-armed bandit problem has been extensively studied \cite{besbes2014stochastic,kalyanakrishnan2012pac,liau2018stochastic}. Under the parametric setting, in which the reward distribution belongs to some family $P_{\theta}$, a type of policy called upper-confidence bound (UCB) is proposed and proved to be promising\cite{auer2002finite}. The policies choose actions based on not only the empirical mean of the rewards but also a UCB. The UCB plays a key role in dealing with the trade-off between exploitation and exploration. Lai and Robbins\cite{lai1985asymptotically} constructed an asymptotically efficient adaptive allocation rule and provided a lower bound for the regret. Agrawal\cite{agrawal1995sample} introduced a family of index policies that is easier to compute. Auer, Nicolo and Paul\cite{auer2002finite} proposed an online, horizon-free procedure which is called upper-confidence bound(UCB) and proved its efficiency. Audibert and Bubeck\cite{MOSS} presented an improvement to the UCB1 called MOSS which is optimal for finite time. A variant of UCB which builds UCB based on the Kullbakc-Leibler divergence(KL) KL-UCB was presented by Garivier, Capp{\'e}\cite{garivier2011kl}. It outperforms the UCB and other variants in practice. Capp{\'e}, Garivier, Maillard, Munos and Stoltz\cite{cappe2013kullback} then proposed kl-UCB$^{+}$ for nonparametric stochastic bandit. Honda and Takemura\cite{honda2015non} presented IMED for the nonparametric stochastic bandit. Later, M{\'e}nard and Garivier\cite{menard2017minimax} proposed the kl-UCB$^{++}$ which is a modified kl-UCB$^{+}$. Lattimore\cite{lattimore2018refining} presented ADA-UCB that is optimal for finite time. The latest UCB algorithm for finite time, KL-UCB-Switch is proposed by Garivier, Hadiji, M{\'e}nard and Stoltz\cite{KLswitch}. In the following sections, we will focus on the parametric and time-horizon free stochastic bandits. 

\subsection{Cold start problem in recommender system}
In the recommender system, the cold start is a situation where the user or content has no previous engagement events, e.g. impressions, clicks, likes, or other interactions. As a closed recommender system, the new users or contents usually start from a cold start cycle: users or contents receive a certain amount of impressions, and then the system will collect the interactions between users and contents for future recommendations. The cold start problem is critical under two major trajectories: i) good cold start policy may increase the major metrics of an application (daily active users (DAU), retention, time spent, etc.) by catching the interest of new users to interact with the application; ii) user-generated content (UGC) is more and more popular so that a mature UGC based recommendation system will receive million to billion new contents which makes recommender system even harder to allocate limited impressions to collect potential feedback for these new contents.

To address the cold start problem, collaborative filtering is a common approach \cite{agarwal2009regression,gunawardana2008tied,park2009pairwise}. However, collaborative filtering often requires a noticeable amount of computation resources such as low-rank matrix factorization \cite{agarwal2009regression}, convex optimization \cite{gunawardana2008tied} or other advanced model like Boltzmann machine \cite{park2009pairwise}. On the other hand, formulating a cold start problem as an MAB problem may utilize the traffic efficiently by considering the exploitation-exploration trade-off \cite{felicio2017multi}. As a promising branch to solve the MAB problem, UCB-type algorithms calculate the confidence bound of all arms and select the arm that has the largest upper confidence bound\cite{auer2002finite, garivier2011kl,cappe2013kullback,li2010contextual,qin2014contextual}. As another bunch of variants, Thompson Sampling algorithms \cite{thompson1933likelihood,agrawal2012analysis,russo2014learning} calculate the probability distribution of each arm with the mean reward value, and then pick the best arm by drawing and comparing values from distributions.

\section{Setup And Notations}
We consider a stochastic multi-armed bandit problem with finite arms $A=\{1,2,3,...K\}$. Each arm $i$ is associated with a reward distribution $p_{i}(\theta)$ over $\mathbb{R}$. It is assumed that $p_{i}(\theta)$ is from some one-parameter exponential family $\mathbb{P(\theta)}$ with unknown expectation $\mu_{i}$.

It is also common to use a different formula for the pseudo-regret for a stochastic problem. Write $\mu^{*}=\max_{i=1,2,...,K}\mathbb{E}[\mu_{i}]$ as the expected reward of the optimal action. Then $\Delta_{i}=\mu_{*}-\mu_{i}$, and we have
\begin{equation}
    \bar{R}_{T}=\sum^{K}_{i=1}{\mathbb{E}[N_{i}(T)]}\mu^{*}-\mathbb{E}[\sum_{i=1}^{K}N_{i}(T)\mu_{i}]=\sum_{i=1}^{K}\Delta_{i}\mathbb{E}[N_{i}(T)]
\end{equation}

The exponential family $P={p(x\arrowvert \theta)}$ is the collection of probability measures whose probability density function has the following general form:
\begin{equation}
    p(x\arrowvert \theta) = \exp(C(x)+\theta^{T}T(x)-\psi(\theta))
\end{equation}
where $\psi(\theta)=\log \int{\exp(C(x)+\theta^{T}T(x))d\mu(x)}$ is called cumulant function, the logarithm of a normalization factor, and $T(x)$ is a sufficient statistics.

In the following sections, the discussion mainly focus on several statistical distances for probability distributions. The most straightforward distance between two distribution is Total Variance distance(TVD):
\begin{equation}
    TVD(P_{0},P_{1})=\frac{1}{2}\|P_{1}-P_{0}\|_{1}=\frac{1}{2}\int_{X}|p_{1}(x)-p_{0}(x)|d\mu
\end{equation}
Another popular statistical distance is Kullback-Leibler divergence (KLD):
\begin{equation}
    D^{(-1)}(P_{0}|P_{1})=KL(P_{0}|P_{1})=\int p_{0}(x)\log \frac{p_{0}(x)}{p_{1}(x)}dx
\end{equation}
Squared Hellinger distance is less popular but has many nice properties. For two distributions from an exponential family, the squared Hellinger distance can be written as
\begin{align}
    H^{2}(P_{\theta_{0}},P_{\theta})=&1-\int{\sqrt{p_{\theta_{0}}p_{\theta}}}dx\\=&1-\int{e^{\frac{1}{2}((\theta_{0}+\theta)T(x)-\psi(\theta_{0})-\psi(\theta)+2C(x))}}dx\\=&1-e^{\psi(\frac{\theta_{0}+\theta}{2})-\frac{\psi(\theta_{0})+\psi(\theta)}{2}}\int{e^{\frac{\theta_{0}+\theta}{2}T(x)-\psi(\frac{\theta_{0}+\theta}{2})+C(x)}}dx\\=&1-e^{\psi(\frac{\theta_{0}+\theta}{2})-\frac{\psi(\theta_{0})+\psi(\theta)}{2}}
\end{align}
Remember that $\dot{\psi}(\theta_{0})=\mathbb{E}[T(x)]$

\section{The Hellinger-UCB Algorithm}
The goal of the UCB algorithm is to make sequential decisions in the stochastic environment. The reward distribution of each arm is unknown. The only way to collect information and estimate the distribution is to pull the arm. But each trial comes with risk which is measured by regret. Hence exploration-exploitation trade-off is important in this case. The motivation of the UCB algorithm is being optimistic about the reward distributions as one always believes that the expected reward is the highest value within the confidence region. Hence the key in the UCB algorithm is constructing the confidence region.

The KLD has some known flaws. It is not a metric since it is usually not symmetric, i.e. $KL(p||q)\neq KL(p||q)$ unless $p=q$. TVD is only upper bounded by the KLD\cite{lehmann2006testing}, 
\begin{equation}
    TVD(p,q)\leqslant\sqrt{\frac{1}{2}KL(P_{0}|P_{1})}
\end{equation}.
While the squared Hellinger distance doesn't have such flaws. It is a metric and controls the Total Variance distance from both sides
\begin{equation}
    H^{2}(p,q)\leqslant\frac{1}{2}TVD(p,q)\leqslant[1-\rho^{2}(p,q)]^{\frac{1}{2}}
\end{equation}.
The formula of the squared Hellinger distance makes it computationally efficient. With this property, we propose a new UCB type algorithm, the Hellinger-UCB which constructs the UCB based on the squared Hellinger distance. The new algorithm achieves the theoretical lower bound and has a closed-form UCB for some distributions, for example, binomial distribution. The latter property is favorable in some low-latency applications. 
\subsection{Main algorithm}
We briefly describe the process of Hellinger-UCB here. Let $A=\{i\}_{i=1}^{K}$ be the action set where $K$, the number of actions, is a positive integer. For each arm $i\in A$, the reward distribution $P_{\theta_{i}}$ is in some one-parameter exponential family with expectation $\mu_{i}$. At the first $|K|$ rounds, the agent chooses each arm once. After that, at each round $t>|K|$, the agent makes a decision $A_{t}=i$ based on the collected observations of each arm and gets the reward $g_{t}(A_{t})$ from $P_{A_{t}}$. The upper confidence bound for arm $i$ is 
\begin{equation}
    U_{i}(t)=\text{sup}\{\dot{\psi}(\theta):H^{2}(P_{\hat{\theta}_{i,t-1}},P_{\theta})\leqslant 1-e^{-c\frac{\log(t)}{N_{i}(t)}}\}
\end{equation}
where $P_{\hat{\theta}_{i,t-1}}$ is the estimated reward distribution based on the past observations and $N_{i}(t)$ the number of pulls of arm $i$. In the right hand side term, $c\in(\frac{1}{4},\frac{1}{2}]$ and usually achieves optimal performance with $c$ slightly greater than $\frac{1}{4}$ in practice. This is a convex optimization problem and can be solved efficiently. The agent will choose the action $i$ with the maximal $U_{i}(t)$. Algorithm 1 shows the pseudo-code of the Hellinger-UCB algorithm.

\begin{algorithm}
  \caption{Hellinger-UCB}
  \begin{enumerate}
    \item
    Known Parameters: $T$(time horizon), $K$(action set), $r_{t}(i)$(reward given action)

    \item
    For $t = 1$ to $|K|$:
    \begin{enumerate}
      \item
      $A_{t}=i=t\%|K|$
      \item
      $N_{i}(t)=1$
      \item
      $S_{i}(t)=r_{t}(A_{t})$
    \end{enumerate}
    end for
    \item
    For $t = |K|+1$ to $T$:
    \begin{enumerate}
      \item
      $A_{t}=\arg{max_{i}\text{sup}\{\dot{\psi}(\theta): H^{2}(P_{\hat{\theta}_{i,t-1}},P_{\theta})\leqslant 1-e^{-c\frac{\log(t)}{N_{i}(t)}}\}}$, where $P_{\hat{\theta}_{t-1}}$ is an estimation of the reward distribution based on the past observations.
      \item
      $r=r_{t}(A_{t})$
      \item
      $N_{i}(t)+=\mathbb{I}\{A_{t}=i\}$
      \item
      $S_{i}(t)+=r$
    \end{enumerate}
    end for
  \end{enumerate}
\end{algorithm}

\subsection{Optimality of Hellinger-UCB}
As a UCB-based algorithm for the stochastic multi-armed bandit problem, we are interested in whether the pseudo regret of the Hellinger-UCB algorithm is optimal. The following theorem guarantees the optimality of this algorithm. We first derive the upper bound of each sub-optimal arm's expected number of pulls.
\begin{theorem}
    Consider a multi-armed bandit problem with $K$ arms and the associated payoffs are some distributions in a one-parameter exponential family. Let $a^{*}$ denote the optimal arm with expectation $\mu^{*}$ and $i$ denote some sub-optimal arm with expectation $\mu_{i}$ such that $\mu_{i}<\mu^{*}$. For any $T>0$, the number of picks of arm $i$ by Hellinger-UCB is $N_{i}(T)$. For any $\epsilon>0$
    \begin{equation*}
        \mathbb{E}[N_{i}(T)]\leqslant-\frac{c\log(T)}{\log(1-\frac{H^{2}(\mu^{*},\mu_{i})}{1+\epsilon})}+\frac{C_{1}(\epsilon)}{T^{C_{2}(\epsilon)}}+\sum_{t=1}^{T}\frac{1}{t^{2c}}+\frac{e^{-2H^{2}(\mu^{*}, \mu_{i})}}{1-e^{-2H^{2}(\mu^{*}, \mu_{i})}}
    \end{equation*}
    where $C_{1}(\epsilon)=-\frac{c}{\log(1-\frac{H^{2}(\mu^{*},\mu_{i})}{1+\epsilon})}>0$ and  $C_{2}(\epsilon)=\frac{(\sqrt{1+\epsilon}-1)^{2}}{1+\epsilon}>0$.\\
    if $c>\frac{1}{4}$,
    \begin{equation*}
        \mathbb{E}[N_{i}(T)]\leqslant-\frac{c\log(T)}{\log(1-\frac{H^{2}(\mu^{*},\mu_{i})}{1+\epsilon})}+\frac{C_{1}(\epsilon)}{T^{C_{2}(\epsilon)}}+O(1)
    \end{equation*}
\end{theorem}
\begin{proof}
See Appendix for details of the proof
\end{proof}

\indent There is a trade-off between two terms $A=\frac{-c\log(T)}{\log(1-\frac{H^{2}(\mu^{*},\mu_{i})}{1+\epsilon})}$ and $B=\sum_{t=1}^{T}\frac{1}{t^{2c}}$. Small $c$ implies fast shrinking Hellinger ball, so it is unlikely to pull sub-optimal arm $i$ if true $\mu^{*}$ is inside the confidence region. But this also implies a lower exploration rate which increases the probability that $\mu^{*}$ is out of the confidence region. Thus small $A$ small but $B$ may diverge faster than $O(\log(T))$; while a large $c$ guarantees converged $B$ which means $\mu^{*}$ is inside its confidence region with high probability. Meanwhile, the Hellinger ball is large which will increase the probability that the confidence region of optimal arm $a^{*}$ is contained in the confidence region of some sub-optimal arm $i$. Thus $A$ is large.

\indent Notice that $\sum_{t=1}^{\infty}\frac{1}{t^{2c}}$ is a $p$-series, it diverges when $c\leqslant\frac{1}{2}$ and converges when $c>\frac{1}{2}$. As a special case of the $p$-series, $\sum_{t=1}^{T}\frac{1}{t}$ is the harmonic series. Although the harmonic series diverges, its rate of divergence is very slow. The partial sums of the harmonic series have logarithmic growth
\begin{equation}
    \log(T+1)<\sum_{t=1}^{T}\frac{1}{t}\leqslant\log(T)+1
\end{equation}
This suggests that if $c\geqslant\frac{1}{2}$, the Hellinger-UCB algorithm achieves an effective upper bound. If we set $c>\frac{1}{4}$, then the upper bound is ineffective but still with the principle term of $O(\log(T))$. In practical application, we suggest using $c$ slightly higher than $\frac{1}{4}$ to achieve better performance. Then the upper bound of the pseudo-regret is just a direct result of Theorem 3.1
\begin{theorem}
    The regret of Hellinger-UCB satisfies:
\begin{equation}
     \bar{R}_{T}\leqslant\sum_{i:\mu_{i}\leqslant\mu^{*}}\Delta_{i} \mathbb{E}[N_{i}(T)]
\end{equation}
\end{theorem}.

\section{NUMERICAL SIMULATION EXAMPLES}
\indent In this section, we will derive the Hellinger-UCB algorithm for several specific reward distributions which all belong to the exponential family. The squared Hellinger distance for each case is also provided. Numerical experiments show that Hellinger-UCB outperforms KL-UCB and achieves lower pseudo-regret\footnote{An open source implementation and comparison can be found in \href{https://github.com/Ruy0216/HellingerUCB}{https://github.com/Ruy0216/HellingerUCB}}.

\subsection{Setting 1: Bernoulli rewards}
We start with the Bernoulli rewards which is a simple discrete distribution that takes $\{0,1\}$ as its support. It has a lot of important applications in MAB problems, e.g. recommender system with the goal of click-through rate (CTR). There are $10$ arms that contain $9$ sub-optimal arms and $1$ optimal arm in the experiment. 
The expected rewards are given by $[0.01, 0.01, 0.01, 0.02, 0.02, 0.02, 0.05, 0.05, 0.05, 0.1]$. 
\begin{figure}[H]
  \centering
  \includegraphics[width=0.8\textwidth,keepaspectratio]{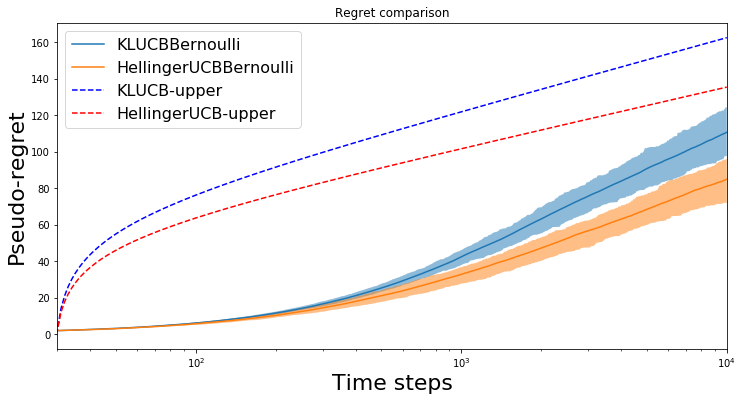}
  \caption{Bernoulli: pseudo regret}
  \label{fig:pseudo regret bernoulli}
\end{figure}
 
\begin{figure}[H]
  \centering
  \includegraphics[width=0.8\textwidth,keepaspectratio]{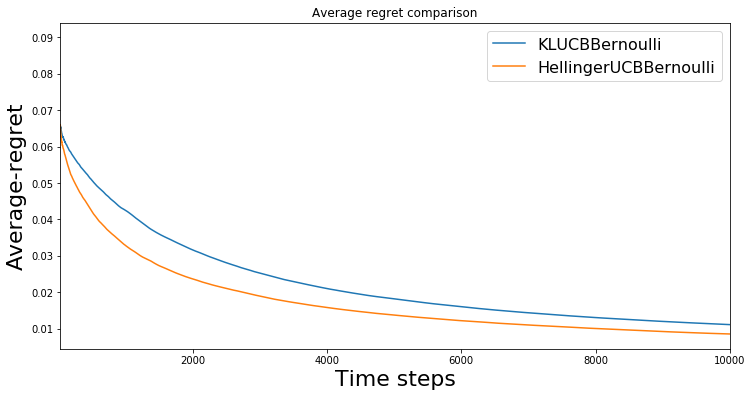}
  \caption{Bernoulli: average pseudo regret}
  \label{fig:average pseudo regret bernoulli}
\end{figure}

\begin{figure}[H]
  \centering
  \includegraphics[width=0.8\textwidth,keepaspectratio]{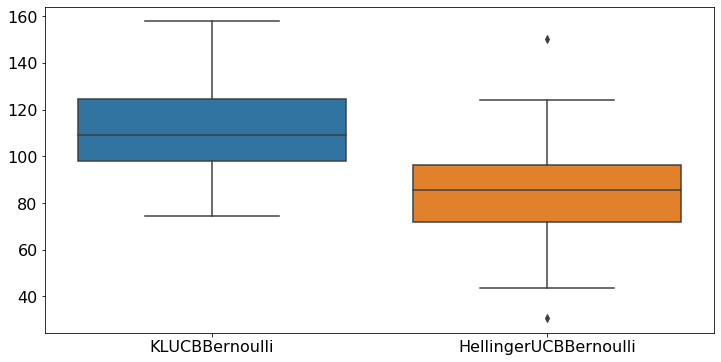}
  \caption{Bernoulli: last step box-plot}
  \label{fig:boxplot bernoulli}
\end{figure}

\subsection{Setting 2: Poisson rewards}
Poisson distribution is often used to model rare events. In this experiment, we set the reward distribution of each arm as a Poisson distribution. We add 6 sub-optimal arms and 1 optimal arm in the experiment and the expected rewards of these arms are given by $[0.03, 0.03, 0.04, 0.04, 0.05, 0.05, 0.1]$.

\begin{figure}[H]
  \centering
  \includegraphics[width=0.8\textwidth,keepaspectratio]{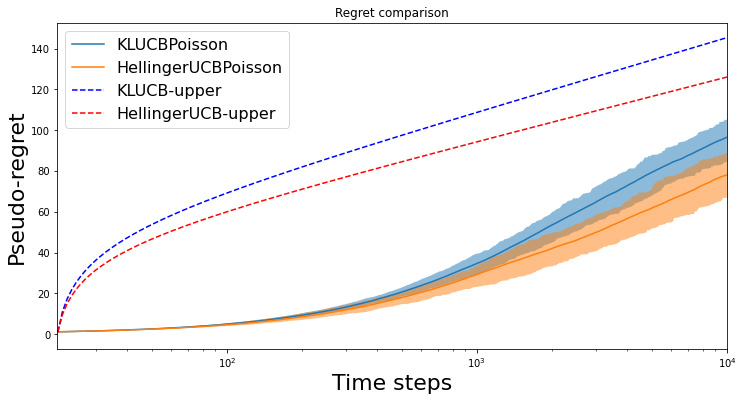}
  \caption{Poisson: pseudo regret}
  \label{fig:pseudo regret poisson}
\end{figure}

\begin{figure}[H]
  \centering
  \includegraphics[width=0.8\textwidth,keepaspectratio]{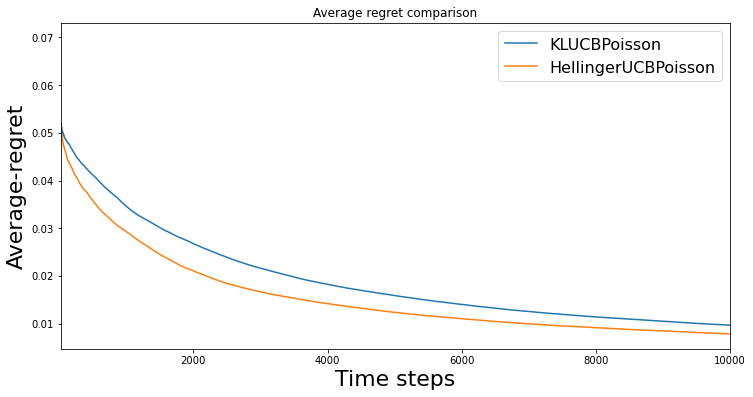}
  \caption{Poisson: average pseudo regret}
  \label{fig:average pseudo regret poisson}
\end{figure}

\begin{figure}[H]
  \centering
  \includegraphics[width=0.8\textwidth,keepaspectratio]{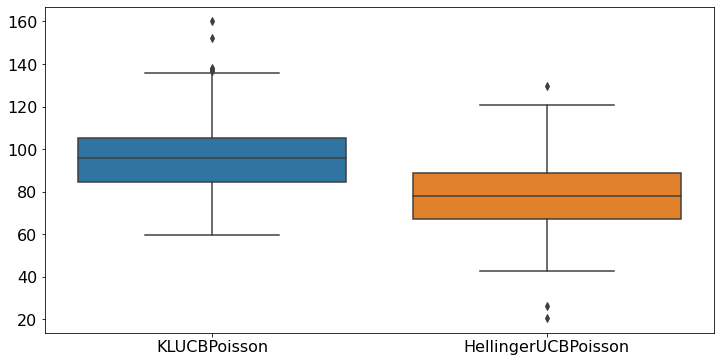}
  \caption{Poisson: last step box-plot}
  \label{fig:boxplot poisson}
\end{figure}

\subsection{Results interpolation}
In Figure \ref{fig:pseudo regret bernoulli} and \ref{fig:pseudo regret poisson}, the solid line is the average pseudo-regret of 200 epochs and the shaded regions are between $25\%$-quantile to $75\%$-quantile. The dashed lines are the pseudo-regret upper bounds. Notice that the x-axis is log-scaled, hence the asymptotically linear property of both pseudo regret curves implies that they are $O(\log(T))$. Clearly, the Hellinger-UCB achieves lower regret than the KL-UCB. Figure \ref{fig:average pseudo regret bernoulli} and \ref{fig:average pseudo regret poisson} show how the average regret changes over time. The gap between the two curves is wider in the early period of the experiment which implies Hellinger-UCB learns faster. Figure \ref{fig:boxplot bernoulli} and \ref{fig:boxplot poisson} are the boxplots of the regret at the last round. It is clear that the Hellinger-UCB outperforms KL-UCB.

\section{Real world example: Article recommendation with cold start}
In this section, we describe a practical example of the Hellinger-UCB algorithm: Article recommendation with cold start. First, we briefly introduce the problem and then we shall show that with reasonable assumption Hellinger-UCB algorithm has a closed-form solution, which is a very nice property in practice. An experiment of reward comparison in the production system is provided at the end of this section. Due to the limitation of traffic and engineering resources, the experiment focuses on the proposed candidate (Hellinger-UCB) and two UCB variants, UCB1 and KL-UCB\footnote{An open source implementation and comparison can be found in \href{https://github.com/guabao/hellinger_ucb_bandits}{https://github.com/guabao/hellinger\_ucb\_bandits}}.

\subsection{Binomial distribution assumption}

Most advanced models in recommender systems use binary feedback signals as main predictions. Some examples are clicks, likes, and comments. Therefore the success and effectiveness of the recommender system are mainly measured by the success rate of these predictions, like click-through rate(CTR) which is the ratio of the number of clicks over the number of impressions. In the real recommender system, the data processing flow records and aggregates the impression data and feedback data on certain resolutions (like every minute or every hour), hence it is convenient to assume that the whole trajectory is a series of Bernoulli trials. With this assumption, the reward of multi-arm bandit setting in the recommender system can be modeled as binomial distribution $B(n, p)$, where $n$ is the number of trails and $p$ is the CTR or other success rate of metrics.

\subsection{The advantage of Hellinger-UCB}\label{assumption}
Hellinger-UCB has two main advantages in the cold start problem. First, Hellinger-UCB has both better numerical stability and performance. In section \ref{UCB-math} we show that under binomial distribution assumption, the Hellinger-UCB algorithm has closed form solution for each arm, and the KL-UCB algorithm has to run an iterative root-finding algorithm to solve an equation with logarithm. The second is production feasibility. Usually, a recommender system is a large system involving the front end, back end, data processing, mode predicting, etc. For a better user experience, this large system should give a recommendation list with relatively low latency (200 milliseconds for news feed or other similar business in general). With fewer feature interactions, the time quota for such a cold start strategy is much less than the whole time limit. For example, the business that runs the Hellinger-UCB algorithm has about 10,000 new contents every day, but the limit of the cold start strategy is only 10 milliseconds. This means the algorithm has to calculate 10,000 arms and pop out top arms in 10 milliseconds. The iterative algorithm has no chance fit such a tiny time limit with reasonable numerical tolerance. But an algorithm with a simple closed-form solution such as Hellinger-UCB can get results on these contents within tiny time limits.

\subsection{Solution of different UCBs} \label{UCB-math}
We will compare the solutions of three UCB algorithms and introduce the advantages and disadvantages of those algorithms. We will also prove that Hellinger-UCB has a close form with binomial distribution assumption of rewards.

\subsubsection{UCB1 algorithm}
The UCB1 algorithm, regardless of the reward distribution, always uses the following UCB formula:
\begin{align}
    U_{t}(i)=\hat{\mu}_{i,t-1}+\sqrt{\frac{2\log(t)}{N_{i}(t)}}
\end{align}
where $\hat{\mu}_{i, t-1}$ is the success ratio of content $i$ at time $t-1$, $N_{i}(t)$ the number of impressions of content $i$ at time $t$. UCB1 algorithm can always compute its confidence bound with a counting process.

\subsubsection{KL-UCB algorithm}
The KL-UCB algorithm solves the following optimization problem numerically to find the best arm
\begin{align}
    U_{t}(i)=\sup\{\mu(\theta):KL(P_{\hat{\theta}_{t-1}}|P_{\theta})\leqslant \frac{\log(t)+c\log\log(t)}{N_{i}(t)}\}
    \label{equation: KL-UCB for cold start}
\end{align}
where $P_{\hat{\theta}_{t-1}}$ and $P_{\hat{\theta}_{t}}$ are the reward distribution at $t-1$ and $t$. With the assumption of reward distribution in \ref{assumption}, the solution of KL-UCB becomes the following root-finding problem
\begin{align}
    U_{t}(i)=\sup\{p_{t}: p_{t-1} \log \frac{p_{t-1}}{p_{t}} + (1- p_{t-1}) \log \frac{1-p_{t-1}}{1-p_{t}} = C\}
    \label{equation: KL-UCB newton equation}
\end{align}
where
\begin{align}
    C = \frac{\log(t)+c\log\log(t)}{N_{i}(t)}
\end{align}
It is easy to get that \ref{equation: KL-UCB newton equation} does not have a closed-form solution and requires a numerical solver to iterate the solution. Therefore KL-USB requires much more computation resources than UCB1 though KL-UCB has better performance than UCB1. In real-world applications, KL-UCB is less favorable than UCB1 since it requires much more careful engineering controls.

\subsubsection{Hellinger-UCB algorithm} The Hellinger-UCB chooses the best arm by solving the following optimization problem
\begin{align}
    U_{t}(i)=\text{sup}\{\dot{\psi}(\theta):H^{2}(P_{\hat{\theta}_{i,t-1}},P_{\theta})\leqslant 1-e^{-c\frac{\log(t)}{N_{i}(t)}}\}
    \label{equation: Hellinger-UCB for cold start}
\end{align}
unlike KL-UCB, the Hellinger-UCB algorithm has a closed-form solution with the binomial reward distribution assumption in \ref{assumption}.

Recall that the squared Hellinger distance between two Binomial distributions $B(n, p)$ and $B(n, q)$ is given by
\begin{equation}
    H^{2}(p,q)=1-\sqrt{(1-p)(1-q)}-\sqrt{pq}
\end{equation}
Let $f(q)=1-\sqrt{(1-p)(1-q)}-\sqrt{pq}$, its derivative is $f^{\prime}(q)=\sqrt{\frac{1-p}{1-q}}-\sqrt{\frac{p}{q}}$. It is easy to see that 
\begin{equation}
    f^{\prime}(q): 
    \begin{cases}
    <0, & q<p,\\
    =0, & q=p,\\
    >0, & q>p.
    \end{cases}
\end{equation}
The solution to the above equation \ref{equation: Hellinger-UCB for cold start} must be on the squared Hellinger ball. Let $R$ be the radius of the squared Hellinger Ball, i.e.
\begin{equation}
    R = 1-\sqrt{(1-p)(1-q)}-\sqrt{pq}.
\end{equation}
Divide both sides by $\sqrt{q}$ and let $m_{1}=\sqrt{\frac{1-p}{p}}$ and $m_{2}=\frac{1-R}{\sqrt{p}}$
\begin{equation}
    \sqrt{q}+m_{1}\sqrt{1-q}=m_{2}.
\end{equation}
Take the square of both sides and simplify the equation
\begin{equation}
    2m_{1}\sqrt{q(1-q)}=m_{2}^{2}-m_{1}^{2}+(m_{1}^{2}-1)q
\end{equation}
Repeat above procedure one more time and simplify the result
\begin{equation}
    (m_{1}^{2}+1)^{2}q^{2}+2(m_{1}^{2}m_{2}^{2}-m_{1}^{4}-m_{1}^{2}-m_{2}^{2})q+(m_{2}^{2}-m_{1}^{2})^{2}=0.
    \label{equation: Hellinger-UCB final equation}
\end{equation}
Let $a=(m_{1}^{2}+1)^{2}$, $b=2(m_{1}^{2}m_{2}^{2}-m_{1}^{4}-m_{1}^{2}-m_{2}^{2})$ and $c=(m_{2}^{2}-m_{1}^{2})^{2}$, the root of \ref{equation: Hellinger-UCB final equation} is 
\begin{equation}
    q=\frac{-b\pm\sqrt{b^{2}-4ac}}{2a}.
\end{equation}
The larger root is desired. Therefore, Hellinger-UCB has a close form solution with binomial reward distribution assumption.

\begin{table}
\begin{center}
\begin{tabular}{ |c||c||c|  }
 \hline
 Algorithm & Close form solution & Reward\\
 \hline
 UCB1   & Yes & low  \\
 \hline
 KL-UCB &  No & moderate  \\
 \hline
 Hellinger-UCB & Yes & high \\
 \hline
\end{tabular}
\end{center}
\caption{Compare the advantage and disadvantage of UCB1, KL-UCB and Hellinger-UCB in real application}
\label{talbe: etf names}
\end{table}

\subsection{Numerical result}
The long-run online experiment is conducted in the front page content recommendation business of JD Finance App. The recommendation system is designed to provide personalized multi-type content recommendations to the users. For each request, the cold start model is required to rank a set of articles and tweets, and then present the top-rank contents to the users. All three algorithms, UCB1, KL-UCB, and Hellinger-UCB, rank about 10 thousand of contents from the cold start pool with estimated CTR. Following the assumption in \ref{assumption}, CTR is modeled as the mean reward of a series of Bernoulli trials, which is the exact historical clicks and impression information. Three UCB algorithms share the whole traffic and the final impression is generated by randomly selecting one of three results uniformly. Therefore three algorithms have relatively the same amount of impressions and do not overlap with each other. The system records 1 point as a reward to the corresponding algorithm when the user has any positive interaction (click/like/comment) with the content. Under this setting, the comparison of rewards among the three algorithms will give an insight into CTR for each algorithm.

Figure \ref{figure: cold start} shows the cumulative reward plot of three algorithms in a two-month experiment from Oct. 2020 to Nov. 2020. It is very clear that Hellinger-UCB (blue line) significantly outperforms KL-UCB (orange line) and UCB1 algorithm (green line). In fact, the Hellinger-UCB algorithm achieves about $33\%$ more clicks than the KL-UCB algorithm and almost $100\%$ more clicks than the UCB1 algorithm. Hellinger-UCB algorithm obtains more clicks in the early period and then achieves even more clicks as the learning continues. This is an encouraging illustration of the potential power of Hellinger-UCB in real applications.

\begin{figure}[H]
  \centering
  \includegraphics[width=0.8\textwidth,keepaspectratio]{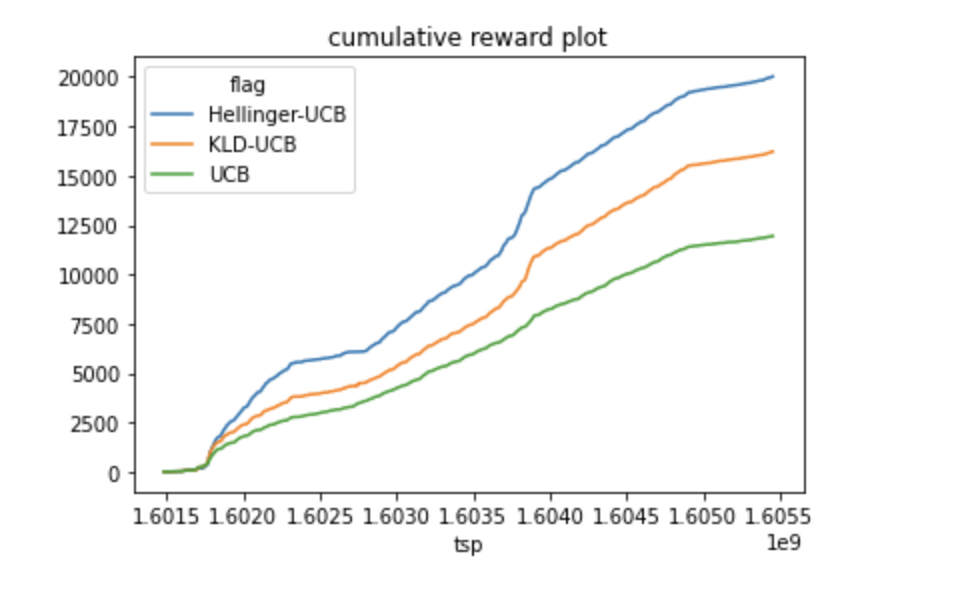}
  \caption{Cumulative reward plot of different UCB algorithms. The y-axis is reward points. The x-axis is a time stamp recorded as a 9-digit integer. The green line is the UCB1 algorithm. The orange line is the KL-UCB algorithm. The blue line is the Hellinger-UCB algorithm. Hellinger-UCB algorithm is the best algorithm of the three. KL-UCB has the second highest reward, and UCB1 has the least reward.}
  \label{figure: cold start}
\end{figure}

\section{Conclusion}
We presented the Hellinger-UCB algorithm for the stochastic multi-armed bandit problem. In the case that the reward is from an unknown exponential family, we provide the detailed formula of the algorithm and an optimal regret upper bound that achieves $O(\log(T))$. We present both simulated and real numerical experiments that show significant improvement over other variants of UCB algorithms. We also show the algorithm has a closed-form UCB when the reward is a bernoulli distribution, which is a beneficial property for low-latency applications.

The squared Hellinger distance is of favorable mathematical properties and statistical interpretation but with less attention. In the future, we plan to apply this idea to other algorithms in different settings, including unparameterized bandits, contextual bandits, etc.

\printbibliography

@article{Robbins1952SomeAO,
  title={Some aspects of the sequential design of experiments},
  author={Herbert E. Robbins},
  journal={Bulletin of the American Mathematical Society},
  year={1952},
  volume={58},
  pages={527-535},
  url={https://api.semanticscholar.org/CorpusID:15556973}
}

@article{lai1985asymptotically,
  title={Asymptotically efficient adaptive allocation rules},
  author={Lai, Tze Leung and Robbins, Herbert},
  journal={Advances in applied mathematics},
  volume={6},
  number={1},
  pages={4--22},
  year={1985},
  publisher={Academic Press}
}

@article{agrawal1995sample,
  title={Sample mean based index policies by o (log n) regret for the multi-armed bandit problem},
  author={Agrawal, Rajeev},
  journal={Advances in applied probability},
  volume={27},
  number={4},
  pages={1054--1078},
  year={1995},
  publisher={Cambridge University Press}
}

@article{auer2002finite,
  title={Finite-time analysis of the multiarmed bandit problem},
  author={Auer, Peter and Cesa-Bianchi, Nicolo and Fischer, Paul},
  journal={Machine learning},
  volume={47},
  pages={235--256},
  year={2002},
  publisher={Springer}
}

@inproceedings{garivier2011kl,
  title={The KL-UCB algorithm for bounded stochastic bandits and beyond},
  author={Garivier, Aur{\'e}lien and Capp{\'e}, Olivier},
  booktitle={Proceedings of the 24th annual conference on learning theory},
  pages={359--376},
  year={2011},
  organization={JMLR Workshop and Conference Proceedings}
}

@article{cappe2013kullback,
  title={Kullback-Leibler upper confidence bounds for optimal sequential allocation},
  author={Capp{\'e}, Olivier and Garivier, Aur{\'e}lien and Maillard, Odalric-Ambrym and Munos, R{\'e}mi and Stoltz, Gilles},
  journal={The Annals of Statistics},
  pages={1516--1541},
  year={2013},
  publisher={JSTOR}
}

@book{lehmann2006testing,
  title={Testing Statistical Hypotheses},
  author={Lehmann, E.L. and Romano, J.P.},
  isbn={9780387276052},
  lccn={2004051464},
  series={Springer Texts in Statistics},
  url={https://books.google.com/books?id=K6t5qn-SEp8C},
  year={2006},
  publisher={Springer New York}
}

@inproceedings{agarwal2009regression,
  title={Regression-based latent factor models},
  author={Agarwal, Deepak and Chen, Bee-Chung},
  booktitle={Proceedings of the 15th ACM SIGKDD international conference on Knowledge discovery and data mining},
  pages={19--28},
  year={2009}
}

@inproceedings{gunawardana2008tied,
  title={Tied boltzmann machines for cold start recommendations},
  author={Gunawardana, Asela and Meek, Christopher},
  booktitle={Proceedings of the 2008 ACM conference on Recommender systems},
  pages={19--26},
  year={2008}
}

@inproceedings{park2009pairwise,
  title={Pairwise preference regression for cold-start recommendation},
  author={Park, Seung-Taek and Chu, Wei},
  booktitle={Proceedings of the third ACM conference on Recommender systems},
  pages={21--28},
  year={2009}
}

@inproceedings{felicio2017multi,
  title={A multi-armed bandit model selection for cold-start user recommendation},
  author={Fel{\'\i}cio, Cr{\'\i}cia Z and Paix{\~a}o, Kl{\'e}risson VR and Barcelos, Celia AZ and Preux, Philippe},
  booktitle={Proceedings of the 25th Conference on User Modeling, Adaptation and Personalization},
  pages={32--40},
  year={2017}
}

@inproceedings{li2010contextual,
  title={A contextual-bandit approach to personalized news article recommendation},
  author={Li, Lihong and Chu, Wei and Langford, John and Schapire, Robert E},
  booktitle={Proceedings of the 19th international conference on World wide web},
  pages={661--670},
  year={2010}
}

@inproceedings{qin2014contextual,
  title={Contextual combinatorial bandit and its application on diversified online recommendation},
  author={Qin, Lijing and Chen, Shouyuan and Zhu, Xiaoyan},
  booktitle={Proceedings of the 2014 SIAM International Conference on Data Mining},
  pages={461--469},
  year={2014},
  organization={SIAM}
}

@inproceedings{MOSS,
  title={Minimax Policies for Adversarial and Stochastic Bandits.},
  author={Audibert, Jean-Yves and Bubeck, S{\'e}bastien and others},
  booktitle={COLT},
  volume={7},
  pages={1--122},
  year={2009}
}

@article{honda2015non,
  title={Non-asymptotic analysis of a new bandit algorithm for semi-bounded rewards.},
  author={Honda, Junya and Takemura, Akimichi},
  journal={J. Mach. Learn. Res.},
  volume={16},
  pages={3721--3756},
  year={2015}
}

@inproceedings{menard2017minimax,
  title={A minimax and asymptotically optimal algorithm for stochastic bandits},
  author={M{\'e}nard, Pierre and Garivier, Aur{\'e}lien},
  booktitle={International Conference on Algorithmic Learning Theory},
  pages={223--237},
  year={2017},
  organization={PMLR}
}

@article{lattimore2018refining,
  title={Refining the confidence level for optimistic bandit strategies},
  author={Lattimore, Tor},
  journal={The Journal of Machine Learning Research},
  volume={19},
  number={1},
  pages={765--796},
  year={2018},
  publisher={JMLR. org}
}

@article{KLswitch,
author = {Garivier, Aur\'{e}lien and Hadiji, H\'{e}di and M\'{e}nard, Pierre and Stoltz, Gilles},
title = {KL-UCB-switch: optimal regret bounds for stochastic bandits from both a distribution-dependent and a distribution-free viewpoints},
year = {2022},
issue_date = {January 2022},
publisher = {JMLR.org},
volume = {23},
number = {1},
issn = {1532-4435},
journal = {J. Mach. Learn. Res.},
month = jan,
articleno = {179},
numpages = {66}
}

@article{thompson1933likelihood,
  title={On the likelihood that one unknown probability exceeds another in view of the evidence of two samples},
  author={Thompson, William R},
  journal={Biometrika},
  volume={25},
  number={3-4},
  pages={285--294},
  year={1933},
  publisher={Oxford University Press}
}

@inproceedings{agrawal2012analysis,
  title={Analysis of thompson sampling for the multi-armed bandit problem},
  author={Agrawal, Shipra and Goyal, Navin},
  booktitle={Conference on learning theory},
  pages={39--1},
  year={2012},
  organization={JMLR Workshop and Conference Proceedings}
}

@article{russo2014learning,
  title={Learning to optimize via posterior sampling},
  author={Russo, Daniel and Van Roy, Benjamin},
  journal={Mathematics of Operations Research},
  volume={39},
  number={4},
  pages={1221--1243},
  year={2014},
  publisher={INFORMS}
}

@article{besbes2014stochastic,
  title={Stochastic multi-armed-bandit problem with non-stationary rewards},
  author={Besbes, Omar and Gur, Yonatan and Zeevi, Assaf},
  journal={Advances in neural information processing systems},
  volume={27},
  year={2014}
}

@inproceedings{kalyanakrishnan2012pac,
  title={PAC subset selection in stochastic multi-armed bandits.},
  author={Kalyanakrishnan, Shivaram and Tewari, Ambuj and Auer, Peter and Stone, Peter},
  booktitle={ICML},
  volume={12},
  pages={655--662},
  year={2012}
}

@inproceedings{liau2018stochastic,
  title={Stochastic multi-armed bandits in constant space},
  author={Liau, David and Song, Zhao and Price, Eric and Yang, Ger},
  booktitle={International Conference on Artificial Intelligence and Statistics},
  pages={386--394},
  year={2018},
  organization={PMLR}
}

@article{burnetas1996optimal,
  title={Optimal adaptive policies for sequential allocation problems},
  author={Burnetas, Apostolos N and Katehakis, Michael N},
  journal={Advances in Applied Mathematics},
  volume={17},
  number={2},
  pages={122--142},
  year={1996},
  publisher={Elsevier}
}

\appendix
\section{Proof of Theorem 3.1}
\begin{proof}
Hellinger-UCB algorithm relies on the following upper confidence bound for $\mu_{i}$:
\begin{equation}
    u_{i}(t)=max\{q>\hat{\mu_{i}}(t):H^{2}(\hat{\mu_{i}}(t),q)\leqslant 1-\exp\{-c\frac{log(t)}{N_{i}(t)}\}\}
\end{equation}
The expectation of $N_{i}(T)$ is upper-bounded by using the following decomposition. When a sub-optimal arm $i$ is pulled, then the upper confidence bound of the optimal arm $u^{*}(t)$ based on historical observations is either greater or less than its true expectation $\mu^{*}$. In the latter case, 
\begin{align}
    \mathbb{E}[N_{i}(T)]&=\mathbb{E}[\sum_{t=1}^{T}\mathbb{I}\{A_{t}=i\}]\\
    &=\mathbb{E}[\sum_{t=1}^{T}\mathbb{I}\{A_{t}=i, \mu^{*}>u^{*}(t)\}]+\mathbb{E}[\sum_{t=1}^{T}\mathbb{I}\{A_{t}=i, \mu^{*}\leqslant u^{*}(t)\}]\\
    &\leqslant \sum_{t=1}^{T}\mathbb{P}\{\mu^{*}>u^{*}(t)\} + \mathbb{E}[\sum_{t=1}^{T}\mathbb{I}\{A_{t}=i, \mu^{*}\leqslant u^{*}(t)\}]\\
    &\leqslant C_{1}(\epsilon)\log(T)+\frac{(C_{2}(\epsilon)H^{2}(\mu^{*},\mu_{i}))^{-1}}{T^{2C_{1}(\epsilon)C_{2}(\epsilon)H^{2}(\mu^{*},\mu_{i})}}\\
    &+\frac{e^{-2H^{2}(\mu^{*}, \mu_{i})}}{1-e^{-2H^{2}(\mu^{*}, \mu_{i})}}+\sum_{t=1}^{T}\frac{1}{t^{2c}}
\end{align}
The last inequality is from Lemma B.1
\begin{equation}
    \sum_{t=1}^{T}\mathbb{P}\{\mu^{*}>u^{*}(t)\}\leqslant\sum_{t=1}^{T}\frac{1}{t^{2c}}
\end{equation}
and Lemma B.3
\begin{align}
    &\mathbb{E}[\sum_{t=1}^{T}\mathbb{I}\{A_{t}=i, \mu^{*}\leqslant u^{*}(t)\}]\\
    \leqslant&C_{1}(\epsilon)\log(T)+\frac{(C_{2}(\epsilon)H^{2}(\mu^{*},\mu_{i}))^{-1}}{T^{2C_{1}(\epsilon)C_{2}(\epsilon)H^{2}(\mu^{*},\mu_{i})}}+\frac{e^{-2H^{2}(\mu^{*}, \mu_{i})}}{1-e^{-2H^{2}(\mu^{*}, \mu_{i})}}
\end{align}
If $c>\frac{1}{4}$, according to Lemma B.2 and Lemma B.3, we have
\begin{align}
   \mathbb{E}[N_{i}(T)]\leqslant C_{1}(\epsilon)\log(T)+\frac{(C_{2}(\epsilon)H^{2}(\mu^{*},\mu_{i}))^{-1}}{T^{2C_{1}(\epsilon)C_{2}(\epsilon)H^{2}(\mu^{*},\mu_{i})}}+O(1)
\end{align}
The details of these lemmas are in the following section.
\end{proof}

\section{THE PROOF OF THE THEOREM}
This concentration inequality\cite{cappe2013kullback} will be used several times
\begin{equation}
    \mathbb{P}\{\hat{\mu}(n)>\mu,KL(\hat{\mu}(n),\mu)>\frac{f(n)}{n}\}\leqslant e^{-f(n)}
\end{equation}.
The following lemmas support the proof of the main theorem.
\begin{lemma}
\begin{equation*}
    \sum_{t=1}^{T}\mathbb{P}\{\mu^{*}>u^{*}(t)\}\leqslant\sum_{t=1}^{T}\frac{1}{t^{2c}}
\end{equation*}
.
\end{lemma}
\begin{proof}
$\hat{\mu}^{*}(t)$ is the M.L.E. of $\mu^{*}$, then
\begin{align}
    &\mathbb{P}\{\mu^{*}>u^{*}(t)\}\\
    \leqslant&\mathbb{P}\{\mu^{*}>\hat{\mu}^{*}(t),H^{2}(\mu^{*},\hat{\mu}^{*}(t))\geqslant1-\exp\{-c\frac{\log(t)}{N^{*}(t)}\}\}
\end{align}
Since for exponential family $-\log(1-H^{2}(\mu^{*},\hat{\mu}^{*}(t)))\leqslant\frac{1}{2}KL(\hat{\mu}^{*}(t),\mu^{*})$, (B.3) becomes
\begin{align}
    &\mathbb{P}\{\mu^{*}>u^{*}(t)\}\\
    \leqslant&\mathbb{P}\{\mu^{*}>\hat{\mu}^{*}(t),1-e^{-\frac{1}{2}KL(\hat{\mu}^{*}(t),\mu^{*})}\geqslant(1-\exp\{-c\frac{\log(t)}{N^{*}(t)}\})\}\\
    \leqslant&\mathbb{P}\{\mu^{*}>\hat{\mu}^{*}(t),KL(\hat{\mu}^{*}(t),\mu^{*})\geqslant2c\frac{\log(t)}{N^{*}(t)}\}\\
    \leqslant&e^{-2c\log(t)}\\
    =&\frac{1}{t^{2c}}
\end{align}.
Then 
\begin{align}
    \sum_{t=1}^{T}\mathbb{P}\{\mu^{*}>u^{*}(t)\}\leqslant\sum_{t=1}^{T}\frac{1}{t^{2c}}
\end{align}

\end{proof}

\begin{lemma}
If $c>\frac{1}{4}$ in
\begin{align*}
    u^{*}(t)=max\{q>\hat{\mu_{i}}(t):H^{2}(\hat{\mu^{*}}(t),q)\leqslant 1-\exp\{-c\frac{log(t)}{N^{*}(t)}\}\}
\end{align*}
then 
\begin{equation*}
    \sum_{t=1}^{\infty}\mathbb{P}\{\mu^{*}>u^{*}(t)\}=O(1)
\end{equation*}
.
\end{lemma}
\begin{proof}
$\hat{\mu}^{*}(t)$ is the M.L.E. of $\mu^{*}$ at $t$
\begin{align}
    &\mathbb{P}\{\mu^{*}>u^{*}(t)\}\\
    \leqslant&\mathbb{P}\{\mu^{*}>\hat{\mu}^{*}(t),H^{2}(\mu^{*},\hat{\mu}^{*}(t))\geqslant1-\exp\{-c\frac{\log(t)}{N^{*}(t)}\}\}
\end{align}
Since $H^{2}(\mu^{*},\hat{\mu}^{*}(t))=\frac{1}{4}\text{KL}(\hat{\mu}^{*}(t),\mu^{*})$ if $t\to\infty$. There exist $T_{1}>0$ and $\delta_{1}>0$ such that for $t>T_{1}$
\begin{align}
    &H^{2}(\mu^{*},\hat{\mu}^{*}(t))\leqslant\frac{1}{4}\text{KL}(\hat{\mu}^{*}(t),\mu^{*})+\delta_{1}
\end{align}
It is known that $\delta_{1}=O(N^{*}(t)^{-\frac{3}{2}})$. Thus (B.11) becomes
\begin{align}
    &\mathbb{P}\{\mu^{*}>u^{*}(t)\}\\
    \leqslant&\mathbb{P}\{\mu^{*}>\hat{\mu}^{*}(t),\text{KL}(\hat{\mu}^{*}(t),\mu^{*})\geqslant4(1-\exp\{-c\frac{\log(t)}{N^{*}(t)}\}-\delta_{1})\}
\end{align}
We have the Taylor expansion 
\begin{align}
    \exp\{-c\frac{\log(t)}{N^{*}(t)}\}&=1-c\frac{\log(t)}{N^{*}(t)}+\frac{1}{2!}(-c\frac{\log(t)}{N^{*}(t)})^{2}+\frac{1}{3!}(-c\frac{\log(t)}{N^{*}(t)})^{3}+...\\
    &=1-c\frac{\log(t)}{N^{*}(t)}-\frac{\log(t)}{N^{*}(t)}R(c,t)
\end{align}
where $R(c,t)=\sum_{k=2}^{\infty}\frac{1}{k!}c^{k}(-\frac{\log(t)}{N^{*}(t)})^{k-1}$ is a negative function for $c<1$. Notice $\lim_{t\to\infty}R(c,t)\to0$ since $\lim_{t\to\infty}\frac{\log(t)}{N^{*(t)}}=0$. There exist $T_{2}>0$ and $\delta>0$ such that $-\delta<R(c,t)<0$. Therefore for $t>\max(T_{1},T_{2})$, (B.14) will be
\begin{align}
    &\mathbb{P}\{\mu^{*}>u^{*}(t)\}\\
    \leqslant&\mathbb{P}\{\mu^{*}>\hat{\mu}^{*}(t),\text{KL}(\hat{\mu}^{*}(t),\mu^{*})\geqslant4c\frac{\log(t)}{N^{*}(t)}+\frac{4\log(t)}{N^{*}(t)}R(c,t)-4\delta_{1}\}\\
    \leqslant&\mathbb{P}\{\mu^{*}>\hat{\mu}^{*}(t),\text{KL}(\hat{\mu}^{*}(t),\mu^{*})\geqslant4c\frac{\log(t)}{N^{*}(t)}-\frac{4\log(t)}{N^{*}(t)}\delta-4\delta_{1}\}\\
    \leqslant&\exp\{-4c\log(t)+4\delta\log(t)+4N^{*}(t)\delta_{1}\}
\end{align}
We now have
\begin{align}
    \mathbb{P}\{\mu^{*}>u^{*}(t)\}\leqslant\frac{e^{4N^{*}(t)\delta_{1}}}{t^{4(c-\delta)}}\leqslant\frac{e^{O(N^{*}(t)^{-\frac{1}{2}})}}{t^{4(c-\delta)}}\leqslant\frac{m_{1}}{t^{4(c-\delta)}}
\end{align}
for some finite $m_{1}>0$.\\
\indent For $t\leqslant\max(T_{1},T_{2})$, there must exist some $m_{2}>0$ such that
\begin{align}
    \mathbb{P}\{\mu^{*}>u^{*}(t)\}\leqslant\frac{m_{2}}{t^{4(c-\delta)}}
\end{align}
Therefore for $M=\max(m_{1},m_{2})$, the following result holds
\begin{align}
    \mathbb{P}\{\mu^{*}>u^{*}(t)\}\leqslant\frac{M}{t^{4(c-\delta)}}
\end{align}
and $c>\frac{1}{4}+\delta$ implies the summation
\begin{align}
    \sum_{t=1}^{\infty}\mathbb{P}\{\mu^{*}>u^{*}(t)\}\leqslant\sum_{t=1}^{\infty}\frac{M}{t^{4(c-\delta)}}=O(1)
\end{align}

\end{proof}

\begin{lemma}
For any $\epsilon>0$, then
\begin{align*}
    &\mathbb{E}[\sum_{t=1}^{T}\mathbb{I}\{A_{t}=i, \mu^{*}\leqslant u^{*}(t)\}]\\
    \leqslant& C_{1}(\epsilon)\log(T)+\frac{(C_{2}(\epsilon)H^{2}(\mu^{*},\mu_{i}))^{-1}}{T^{2C_{1}(\epsilon)C_{2}(\epsilon)H^{2}(\mu^{*},\mu_{i})}}+\frac{e^{-2H^{2}(\mu^{*}, \mu_{i})}}{1-e^{-2H^{2}(\mu^{*}, \mu_{i})}}
\end{align*}
where $C_{1}(\epsilon)=-\frac{c}{\log(1-\frac{H^{2}(\mu^{*},\mu_{i})}{1+\epsilon})}>0$ and $C_{2}(\epsilon)=\frac{(\sqrt{1+\epsilon}-1)^{2}}{1+\epsilon}>0$.
\end{lemma}
\begin{proof}
Arm $i$ is sub-optimal with expected reward $\mu_{i}$ and $\hat{\mu}_{i}(t)$ is the M.L.E. for $\mu_{i}$ at $t$,. Then we have
\begin{align}
    &\mathbb{E}[\sum_{t=1}^{T}\mathbb{I}\{A_{t}=i, \mu^{*}\leqslant u^{*}(t)\}]\\
    =&\mathbb{E}[\sum_{t=1}^{T}\mathbb{I}\{A_{t}=i, \mu^{*}\leqslant u^{*}(t), \mu^{*}\leqslant\hat{\mu}_{i}(t)\}]+\\
    &\mathbb{E}[\sum_{t=1}^{T}\mathbb{I}\{A_{t}=i, \mu^{*}\leqslant u^{*}(t), \mu^{*}>\hat{\mu}_{i}(t)\}]\\
    &\leqslant C_{1}(\epsilon)\log(T)+\frac{(C_{2}(\epsilon)H^{2}(\mu^{*},\mu_{i}))^{-1}}{T^{2C_{1}(\epsilon)C_{2}(\epsilon)H^{2}(\mu^{*},\mu_{i})}}+\frac{e^{-2H^{2}(\mu^{*}, \mu_{i})}}{1-e^{-2H^{2}(\mu^{*}, \mu_{i})}}
\end{align}
(B.28) is according to Lemma B.4 and Lemma B.5.

\end{proof}

\begin{lemma}
\begin{align*}
    \mathbb{E}[\sum_{t=1}^{T}\mathbb{I}\{A_{t}=i, \mu^{*}\leqslant u^{*}(t), \mu^{*}\leqslant\hat{\mu}_{i}(t)\}]
    \leqslant\frac{e^{-2H^{2}(\mu^{*}, \mu_{i})}(1-e^{-2TH^{2}(\mu^{*}, \mu_{i})})}{1-e^{-2H^{2}(\mu^{*}, \mu_{i})}}
\end{align*}
\end{lemma}
\begin{proof}
\begin{align}
    &\mathbb{E}[\sum_{t=1}^{T}\mathbb{I}\{A_{t}=i, \mu^{*}\leqslant u^{*}(t), \mu^{*}\leqslant\hat{\mu}_{i}(t)\}]\\
    \leqslant&\mathbb{E}[\sum_{t=1}^{T}\mathbb{I}\{A_{t}=i, \mu^{*}\leqslant\hat{\mu}_{i}(t), H^{2}(\mu^{*}, \mu_{i})\leqslant H_{2}(\hat{\mu}_{i}(t), \mu_{i})\}]\\
    \leqslant&\mathbb{E}[\sum_{t=1}^{T}\mathbb{I}\{A_{t}=i, \mu^{*}\leqslant\hat{\mu}_{i}(t), H^{2}(\mu^{*}, \mu_{i})\leqslant \frac{1}{2}KL(\hat{\mu}_{i}(t), \mu_{i})\}]\\
    =&\mathbb{E}[\sum_{t=1}^{T}\sum_{s=1}^{t}\mathbb{I}\{A_{t}=i, N_{i}(t)=s, \mu^{*}\leqslant\hat{\mu}_{i}(s), H^{2}(\mu^{*}, \mu_{i})\leqslant \frac{1}{2}KL(\hat{\mu}_{i}(s), \mu_{i})\}]\\
    =&\mathbb{E}[\sum_{s=1}^{T}\sum_{t=s}^{T}\mathbb{I}\{A_{t}=i, N_{i}(t)=s\}\mathbb{I}\{\mu^{*}\leqslant\hat{\mu}_{i}(s), H^{2}(\mu^{*}, \mu_{i})\leqslant \frac{1}{2}KL(\hat{\mu}_{i}(s), \mu_{i})\}]\\
    =&\mathbb{E}[\sum_{s=1}^{T}\mathbb{I}\{\mu^{*}\leqslant\hat{\mu}_{i}(s), H^{2}(\mu^{*}, \mu_{i})\leqslant \frac{1}{2}KL(\hat{\mu}_{i}(s), \mu_{i})\}\sum_{t=s}^{T}\mathbb{I}\{A_{t}=i, N_{i}(t)=s\}]
\end{align}
Notice in (B.34) $\sum_{t=s}^{T}\mathbb{I}\{A_{t}=i, N_{i}(t)=s\}]\leqslant1$ and thus 
\begin{align}
    &\mathbb{E}[\sum_{t=1}^{T}\mathbb{I}\{A_{t}=i, \mu^{*}\leqslant u^{*}(t), \mu^{*}\leqslant\hat{\mu}_{i}(t)\}]\\
    \leqslant&\mathbb{E}[\sum_{s=1}^{T}\mathbb{I}\{\mu^{*}\leqslant\hat{\mu}_{i}(s), H^{2}(\mu^{*}, \mu_{i})\leqslant \frac{1}{2}KL(\hat{\mu}_{i}(s), \mu_{i})\}]\\
    =&\sum_{s=1}^{T}\mathbb{P}\{\mu^{*}\leqslant\hat{\mu}_{i}(s), H^{2}(\mu^{*}, \mu_{i})\leqslant \frac{1}{2}KL(\hat{\mu}_{i}(s), \mu_{i})\}\\
    \leqslant&\sum_{s=1}^{T}e^{-2sH^{2}(\mu^{*}, \mu_{i})}\\
    =&\frac{e^{-2H^{2}(\mu^{*}, \mu_{i})}(1-e^{-2TH^{2}(\mu^{*}, \mu_{i})})}{1-e^{-2H^{2}(\mu^{*}, \mu_{i})}}
\end{align}
It is easy to show
\begin{align}
    \lim_{T\to\infty}\mathbb{E}[\sum_{t=1}^{T}\mathbb{I}\{A_{t}=i, \mu^{*}\leqslant u^{*}(t), \mu^{*}\leqslant\hat{\mu}_{i}(t)\}]=\frac{e^{-2H^{2}(\mu^{*}, \mu_{i})}}{1-e^{-2H^{2}(\mu^{*}, \mu_{i})}}
\end{align}
which is a problem-dependent constant.
\end{proof}

\begin{lemma}
For any $\epsilon>0$, then 
\begin{align*}
    \mathbb{E}[\sum_{t=1}^{T}\mathbb{I}\{A_{t}=i, \mu^{*}\leqslant u^{*}(t), \mu^{*}>\hat{\mu}_{i}(t)\}]\leqslant C_{1}(\epsilon)\log(T)+\frac{(C_{2}(\epsilon)H^{2}(\mu^{*},\mu_{i}))^{-1}}{T^{2C_{1}(\epsilon)C_{2}(\epsilon)H^{2}(\mu^{*},\mu_{i})}}
\end{align*}
\raggedright where $C_{1}(\epsilon)=-\frac{c}{\log(1-\frac{H^{2}(\mu^{*},\mu_{i})}{1+\epsilon})}>0$ and $C_{2}(\epsilon)=\frac{(\sqrt{1+\epsilon}-1)^{2}}{1+\epsilon}>0$.\\
\end{lemma}
\begin{proof}
Similar to the proof for Lemma B.4 we can have 
\begin{align}
    &\mathbb{E}[\sum_{t=1}^{T}\mathbb{I}\{A_{t}=i, \mu^{*}\leqslant u^{*}(t), \mu^{*}>\hat{\mu}_{i}(t)\}]\\
    \leqslant&\mathbb{E}[\sum_{t=1}^{T}\sum_{s=1}^{t}\mathbb{I}\{A_{t}=i, N_{i}(t)=s, H^{2}(\mu^{*}, \hat{\mu}_{i}(s))<1-e^{-c\frac{\log(t)}{s}}\}]\\
    \leqslant&\mathbb{E}[\sum_{s=1}^{T}\sum_{t=s}^{T}\mathbb{I}\{A_{t}=i, N_{i}(t)=s\}\mathbb{I}\{H^{2}(\mu^{*}, \hat{\mu}_{i}(s))<1-e^{-c\frac{\log(T)}{s}}\}]\\
    \leqslant&\mathbb{E}[\sum_{s=1}^{T}\mathbb{I}\{H^{2}(\mu^{*}, \hat{\mu}_{i}(s))<1-e^{-c\frac{\log(T)}{s}}\}]\\
    =&\sum_{s=1}^{T}\mathbb{P}\{H^{2}(\mu^{*}, \hat{\mu}_{i}(s))<1-e^{-c\frac{\log(T)}{s}}\}\\
    \leqslant&K_{T}+\sum_{s=K_{T}+1}^{T}\mathbb{P}\{H^{2}(\mu^{*}, \hat{\mu}_{i}(s))<1-e^{-c\frac{\log(T)}{K_{T}}}\}
\end{align}
where $K_{T}=C_{1}(\epsilon)\log{(T)}$ and $C_{1}(\epsilon)=-\frac{c}{\log(1-\frac{H^{2}(\mu^{*},\mu_{i})}{1+\epsilon})}>0$. Then substitute this into (B.46). Then
\begin{align}
    &\mathbb{E}[\sum_{t=1}^{T}\mathbb{I}\{A_{t}=i, \mu^{*}\leqslant u^{*}(t), \mu^{*}>\hat{\mu}_{i}(t)\}]\\
    \leqslant&C_{1}(\epsilon)\log(T)+\sum_{s=K_{T}+1}^{T}\mathbb{P}\{H^{2}(\mu^{*}, \hat{\mu}_{i}(s))<\frac{H^{2}(\mu^{*},\mu_{i})}{1+\epsilon}\}
\end{align}
There exist $\mu^{\prime}\in(\mu_{i},\mu^{*})$ such that $(1+\epsilon)H^{2}(\mu^{*},\mu^{\prime})=H^{2}(\mu^{*},\mu_{i})$. Then $H^{2}(\mu^{*}, \hat{\mu}_{i}(s))<\frac{H^{2}(\mu^{*},\mu_{i})}{1+\epsilon}$ implies $\mu_{i}<\mu^{\prime}<\hat{\mu}_{i}(s)$ and $H^{2}(\mu_{i},\hat{\mu}_{i}(s))> H^{2}(\mu_{i},\mu^{\prime})$. The second term in (B.48) becomes
\begin{align}
    &\sum_{s=K_{T}+1}^{\infty}\mathbb{P}\{H^{2}(\mu^{*}, \hat{\mu}_{i}(s))<\frac{H^{2}(\mu^{*},\mu_{i})}{1+\epsilon}\}\\
    \leqslant&\sum_{s=K_{T}+1}^{\infty}\mathbb{P}\{\hat{\mu}_{i}(s)>\mu_{i}, H^{2}(\mu_{i},\hat{\mu}_{i}(s))>H^{2}(\mu_{i},\mu^{\prime})\}\\
    \leqslant&\sum_{s=K_{T}+1}^{\infty}\mathbb{P}\{\hat{\mu}_{i}(s)>\mu_{i}, \frac{1}{2}KL(\hat{\mu}_{i}(s),\mu_{i})>H^{2}(\mu_{i},\mu^{\prime})\}\\
    \leqslant&\sum_{s=K_{T}+1}^{\infty}e^{-2sH^{2}(\mu_{i},\mu^{\prime})}\\
    =&\frac{e^{-2(K_{T}+1)H^{2}(\mu_{i},\mu^{\prime})}}{1-e^{-2H^{2}(\mu_{i},\mu^{\prime})}}
\end{align}   
The numerator of (B.53) $e^{-2(K_{T}+1)H^{2}(\mu_{i},\mu^{\prime})}\leqslant T^{-2C_{1}(\epsilon)H^{2}(\mu_{i},\mu_{\prime})}$. For the denominator of (B.53), we have $1-e^{-2H^{2}(\mu_{i},\mu^{\prime})}>H^{2}(\mu_{i},\mu^{\prime})$ since $1-e^{-x}=1-(1-x+\frac{x^{2}}{2}-...)>x-\frac{x^{2}}{2}>x-\frac{x}{2}$ if $0<x<1$. Therefore
\begin{align}
    \sum_{s=K_{T}+1}^{\infty}\mathbb{P}\{H^{2}(\mu^{*}, \hat{\mu}_{i}(s))<\frac{H^{2}(\mu^{*},\mu_{i})}{1+\epsilon}\}\leqslant\frac{H^{2}(\mu_{i},\mu^{\prime})^{-1}}{T^{2C_{1}(\epsilon)H^{2}(\mu_{i},\mu_{\prime})}}
\end{align}
Since the squared Hellinger distance is a metric, we have
\begin{align}
    H(\mu^{*},\mu_{i})=&\sqrt{1+\epsilon}H((\mu^{*},\mu^{\prime}))\\
    \geqslant&\sqrt{1+\epsilon}(H(\mu^{*},\mu_{i})-H(\mu^{\prime},\mu_{i}))
\end{align}
This implies
\begin{align}
    H^{2}(\mu^{\prime},\mu_{i})\geqslant\frac{(\sqrt{1+\epsilon}-1)^{2}}{1+\epsilon}H^{2}(\mu^{*},\mu_{i})=C_{2}(\epsilon)H^{2}(\mu^{*},\mu_{i})
\end{align}
Finally, we conclude that
\begin{align}
    \mathbb{E}[\sum_{t=1}^{T}\mathbb{I}\{A_{t}=i, \mu^{*}\leqslant u^{*}(t), \mu^{*}>\hat{\mu}_{i}(t)\}]\leqslant\frac{(C_{2}(\epsilon)H^{2}(\mu^{*},\mu_{i}))^{-1}}{T^{2C_{1}(\epsilon)C_{2}(\epsilon)H^{2}(\mu^{*},\mu_{i})}}
\end{align}
where $C_{1}(\epsilon)>0$ and $C_{2}(\epsilon)>0$.
\end{proof}

\end{document}